\documentclass[a4paper]{article}

\usepackage{graphicx,amsthm}

\usepackage{algpseudocode,comment,wrapfig}
\usepackage{nicefrac}
\usepackage[linesnumbered,ruled]{algorithm2e}
\usepackage{color}
\usepackage{caption}
\usepackage{subcaption}
\usepackage{epsfig,amssymb,amsmath,amsthm,amsfonts}

\usepackage[draft]{changes}
\definechangesauthor[name=vemuri,color=blue]{BV}

\numberwithin{equation}{section}
\theoremstyle{plain}
\newtheorem{theorem}{Theorem}[section]
\newtheorem{lemma}{Lemma}[section]
\newtheorem{proposition}{Proposition}[section]
\DeclareMathOperator*{\argmin}{arg\,min}

\theoremstyle{plain}
\newtheorem{definition}{Definition}[section]

\newcommand{\at}[2][]{#1|_{#2}}

\newcommand{{\bL}}{{\mathbf {L}}}
\newcommand{{\CAT}}{{\mathbf {CAT}}}
\newcommand{{\PP}}{{\mathbb {P}}}
\newcommand{{\XX}}{{\mathbf {X}}}
\newcommand{{\RR}}{{\mathbb {R}}}
\newcommand{{\bP}}{{\mathbf {P}}}
\newcommand{{\bbP}}{{\mathbb {P}}}
\newcommand{{\cM}}{{\mathcal {M}}}
\newcommand{{\bM}}{{\mathbf {M}}}
\newcommand{{\bN}}{{\mathbf {N}}}
\newcommand{{\bD}}{{\mathbf {D}}}
\newcommand{{\bI}}{{\mathbf {I}}}
\newcommand{{\bg}}{{\mathbf {g}}}
\newcommand{{\Trace}}{{\mathbf {tr}}}
\newcommand{{\dd}}{{\mathbf {d}}}
\newcommand{{\Exp}}{{\mathbf {Exp}}}
\newcommand{{\Var}}{{\mathbf {Var}}}
\newcommand{{\MSE}}{{\mathbf {MSE}}}
\newcommand{{\Log}}{{\mathbf {Log}}}
\newcommand{{\Logm}}{{\mathbf {Logm}}}
\newcommand{{\bR}}{{\mathbf {R}}}
\newcommand{{\bS}}{{\mathbf {S}}}
\newcommand{{\bT}}{{\mathbf {T}}}
\newcommand{{\bX}}{{\mathbf {X}}}
\newcommand{{\by}}{{\mathbf {y}}}
\newcommand{{\bx}}{{\mathbf {x}}}
\newcommand{{\GL}}{{\mathbf {GL}}}
\newcommand{{\bm}}{{\mathbf {m}}}
\newcommand{{\bz}}{{\mathbf {z}}}
\newcommand{{\bd}}{{\mathbf {d}}}
\newcommand{{\bV}}{{\mathbf {V}}}
\newcommand{{\bU}}{{\mathbf {U}}}
\newcommand{{\bO}}{{\mathbf {O}}}
\newcommand{{\bo}}{{\mathbf {o}}}
\newcommand{{\bp}}{{\mathbf {p}}}
\newcommand{{\bq}}{{\mathbf {q}}}
\newcommand{{\br}}{{\mathbf {r}}}

\begin{document}

%\begin{frontmatter}
\title{Statistics on the (compact) Stiefel manifold: Theory and Applications}
%\runtitle{Statistics on the (compact) Stiefel manifold}
\author{
Rudrasis Chakraborty\textsuperscript{1} and Baba C. Vemuri\textsuperscript{1} \\
\textsuperscript{1}Department of CISE, University of Florida, FL 32611, USA\\
\textsuperscript{1}{\tt\small \{rudrasischa, baba.vemuri\}@gmail.com} 
}
\date{}
\maketitle

%\end{aug}

\begin{abstract}
A Stiefel manifold of the compact type is often encountered in many
fields of Engineering including, signal and image processing, machine
learning, numerical optimization and others. The Stiefel manifold is a
Riemannian homogeneous space but not a symmetric space. In previous
work, researchers have defined probability distributions on symmetric
spaces and performed statistical analysis of data residing in these
spaces. In this paper, we present original work involving definition
of Gaussian distributions on a homogeneous space and show that the
maximum-likelihood estimate of the location parameter of a Gaussian
distribution on the homogeneous space yields the Fr\'{e}chet mean (FM)
of the samples drawn from this distribution. Further, we present an
algorithm to sample from the Gaussian distribution on the Stiefel
manifold and recursively compute the FM of these samples. We also
prove the weak consistency of this recursive FM estimator. Several
synthetic and real data experiments are then presented, demonstrating
the superior computational performance of this estimator over the
gradient descent based non-recursive counter part as well as the
stochastic gradient descent based method prevalent in literature.

\end{abstract}

\section{Introduction}\label{intro}

Manifold-valued data have gained much importance in recent times due
to their expressiveness and ready availability of machines with
powerful CPUs and large storage. For example, these data arise as {\it
  rank-2 tensors} (manifold of symmetric positive definite matrices)
\cite{moakher2006averaging,pennec2006riemannian}, {\it linear
  subspaces} (the Grassmann manifold)
\cite{turaga2008statistical,hauberg2014grassmann,goodall1999projective,patrangenaru2003affine},
{\it column orthogonal matrices} (the Stiefel manifold)
\cite{turaga2008statistical,hendriks1998mean,chikuse1991asymptotic},
{\it directional data} and {\it probability densities} (the
hypersphere)
\cite{directional_data,Anuj_cvpr07,Tuch_2003,Hartley_IJCV13}
and others. A useful method of analyzing manifold valued
data is to compute statistics on the underlying
manifold. The most popular statistic is a {\it summary} of
the data, i.e., {\it the} Riemannian barycenter (Fr\'{e}chet
mean (FM))
\cite{frechet1948elements,karcher1977riemannian,afsari},
Fr\'{e}chet median
\cite{arnaudon2013riemannian,charfi2013bhattacharyya} etc.
However, in order to compute statistics of manifold-valued data, the 
first step involves defining a distribution on the manifold. Recently,
authors in \cite{saidvemuri} have defined a Gaussian distribution on
Riemannian symmetric spaces (or symmetric spaces). Some typical
examples of symmetric spaces include the Grassmannian, the hypersphere
etc. Several other researchers
\cite{cheng2013novel,said2015riemannian} have defined a Gaussian
distribution on the space of symmetric positive definite
matrices. They called the distribution a ``generalized Gaussian
distribution'' \cite{cheng2013novel} and ``Riemannian Gaussian
distribution'' \cite{said2015riemannian} respectively. 

In this work, we define a Gaussian distribution on a homogeneous space
(a more general class than symmetric spaces). A key difficulty in
defining the Gaussian distribution on a non-Euclidean space is to show
that the normalizing factor in the expression for the distribution is
a constant. In this work, we show that the normalizing factor in our
definition of the Gaussian distribution on a homogeneous space is
indeed a constant. Note that a symmetric space is a homogeneous space
but not all homogeneous spaces are symmetric and thus, our definition
of Gaussian distribution is on a more generalized topological space
than the symmetric space. Given a well-defined Gaussian distribution,
the next step is to estimate the parameters of the distribution. In
this work, we prove that the maximum likelihood estimate (MLE) of the
mean of the Gaussian distribution is the Fr\'{e}chet mean (FM) of the
samples drawn from the distribution.

Data with values in the space of column orthogonal matrices have
become popular in many applications of Computer Vision and Medical
Image analysis
\cite{turaga2008statistical,rudra2017isbi,lui2012advances,cetingul2009intrinsic,pham2008robust}. The
space of column orthogonal matrices is a topological space, and
moreover one can equip this space with a Riemannian metric which in
turn makes this space a Riemannian manifold, known as the Stiefel
manifold. The Stiefel manifold is a homogeneous space and here we
extend the definition of the Gaussian distribution to the Stiefel
manifold. In this work, we restrict ourselves to the Stiefel manifold
of the compact type, which is quite commonly encountered in most
applications mentioned earlier.

We now motivate the need for a recursive FM estimator. In this age of
massive and continuous streaming data, samples are often acquired
incrementally.  Hence, from an applications perspective, the desired
algorithm should be {\it recursive/inductive} in order to maximize
computational efficiency and account for availability of data,
requirements that are seldom addressed in more theoretically oriented
fields. We propose an {\it inductive} FM computation algorithm and
prove the weak consistency of our proposed estimator. FM computation
on Riemannian manifolds has been an active area of research for the
past few decades. Several researchers have addressed this problem and
we refer the reader to
\cite{bhattacharya2008statistics,afsari,groisser2004newton,pennec_jmiv06,ando2004geometric,moakher2005differential,bhatia2013matrix,Rao1987,fletcher2007riemannian,arnaudon2013riemannian,sturm2003probability,ho2013recursive,Chakraborty_2015_ICCV,salehian2015efficient}.

\subsection{Key Contributions}
In summary, the key contributions of this paper are: (i) A novel
generalization of Gaussian distributions to homogeneous
spaces. (ii) A proof that the MLE of the location parameter of this
distribution is ``the'' FM. (iii) A sampling technique for drawing
samples from this generalized Gaussian distribution defined on a
compact Stiefel manifold (which is a homogeneous space), and an
inductive/recursive FM estimator from the drawn samples along with a
proof of its weak consistency. Several examples of FM estimates
computed from real and synthetic data are shown to illustrate the
power of the proposed methods.

Though researchers have defined Gaussian distributions
  on other manifolds in the past, see
  \cite{saidvemuri,cheng2013novel}, their generalization of the
  Gaussian distribution is restricted to symmetric spaces of non-compact types. In this
  work, we define a Gaussian distribution on a homogeneous space,
  which is a more general topological space than the symmetric
  space. A few others in literature have generalized the Gaussian
  distribution to all Riemannian manifolds, for instance, in
  \cite{zhang2013probabilistic}, authors defined the Gaussian
  distribution on a Riemannian manifold without a proof to show that
  the normalizing factor is a constant. Whereas, in
  \cite{pennec_jmiv06}, the author defined the normal law on
  Riemannian manifolds using the concept of entropy maximization for
  distributions with known mean and covariance. Under certain
  assumptions, the author shows that this definition amounts to using
  the Riemannian exponential map on a truncated Gaussian distribution
  defined in the tangent space at the known intrinsic mean.  This
  approach of deriving the normal distribution yields a normalizing
  factor that is dependent on the location parameter of the
  distribution and hence is not a constant with respect to the FM.  To
  the best of our knowledge, we are the first to define a Gaussian
  distribution on a homogeneous space with a constant normalizing factor, i.e., the normalizing factor does not depend on the location parameter (FM).  

We then move our focus to the Stiefel manifold (which is a homogeneous
space) and propose a simple algorithm to draw samples from the
Gaussian distribution on the Stiefel manifold. In order to achieve
this, we develop a simple but non-trivial way to extend the sampling
algorithm in \cite{saidvemuri} to get samples on the Stiefel
manifold. Once we have the samples from a Gaussian distribution on the
Stiefel, we propose a novel estimator of the sample FM and prove the
weak consistency of this estimator. The proposed FM estimator is
inductive in nature and is motivated by the inductive FM algorithm on
the Euclidean space. But, unlike Euclidean space, due to the presence
of non-zero curvature, it is necessary to prove the consistency of our
proposed estimator, which is presented subsequently. Further, we
experimentally validate the superior performance of our proposed FM
estimator over the gradient descent based techniques. Moreover, we
also show that the MLE of the location parameter of the Gaussian
distribution on the Stiefel manifold asymptotically achieves the
Cram\'er-Rao lower bound
\cite{cramer2016mathematical,rao1992information}, hence in turn, the
MLE of the location parameter is efficient. This implies that our
proposed consistent FM estimator, asymptotically, has a variance lower
bounded by that of the MLE. 

The rest of the paper is organized as follows. In section
\ref{theory}, we present the necessary mathematical background.  In
section \ref{sampling}, we define a Gaussian distribution on a
homogeneous space. More specifically, define a generalized Gaussian
distribution on the Stiefel manifold and prove that the normalizing
factor is indeed a constant with respect to the location parameter of
the distribution. Then, we propose a sampling algorithm to draw
samples from this generalized Gaussian distribution in section
\ref{sampling1} and in section \ref{sampling2}, show that the MLE of
the location parameter of this Gaussian distribution is the FM of the
samples drawn from the distribution. In section \ref{sec2}, we propose
an inductive FM estimator and prove its weak consistency.  Finally we
present a set of synthetic and real data experiments in section
\ref{sec3} and draw conclusions in section \ref{conc}.

\section{Mathematical Background: Homogeneous spaces and the Riemannian symmetric space} \label{theory}

In this section, we present a brief note on the differential geometry
background required in the rest of the paper. For a detailed
exposition on these concepts, we refer the reader to a comprehensive
and excellent treatise on this topic by Helgason
\cite{helgason}. Several propositions and lemmas that
  are needed to prove the results in the rest of the paper are stated
  and proved here. Some of these might have been presented in the vast
  differential geometry literature but are unknown to us and hence the
  proofs presented in this background section are original.

Let $(\mathcal{M},g^\mathcal{M})$ be a Riemannian manifold with a
Riemannian metric $g^\mathcal{M}$, i.e., $(\forall x \in
\mathcal{M})\:g^\mathcal{M}_x:T_x{\mathcal{M}}\times T_x{\mathcal{M}}
\rightarrow \mathbf{R}$ is a bi-linear symmetric positive definite
map, where $T_x\mathcal{M}$ is the tangent space of $\mathcal{M}$ at
$x\in \mathcal{M}$. Let $d: \mathcal{M} \times \mathcal{M} \rightarrow
\mathbf{R}$ be the metric (distance) induced by the Riemannian metric
$g^\mathcal{M}$.  Let $I(\mathcal{M})$ be the set of all isometries of
$\mathcal{M}$, i.e., given $g \in I(\mathcal{M})$, $d(g.x, g.y) = d(x,
y)$, for all $x, y \in \mathcal{M}$. It is clear that $I(\mathcal{M})$
forms a group (henceforth, we will denote $I(\mathcal{M})$ by $(G,
\cdot)$) and thus, for a given $g \in G$ and $x \in \mathcal{M}$, $g.x
\mapsto y$, for some $y \in \mathcal{M}$ is a group action. Consider
$o \in \mathcal{M}$, and let $H = \text{Stab}(o) = \{h \in G | h.o =
o\}$, i.e., $H$ is the {\it Stabilizer} of $o \in \mathcal{M}$.
%\begin{proposition}
%$H \triangleleft G$, i.e., $H$ is a normal subgroup of $G$.
%\end{proposition}
We say that $G$ acts {\it transitively} on $\mathcal{M}$,
\emph{iff}, given $x, y \in \mathcal{M}$, there exists a $g \in
\mathcal{M}$ such that $y=g.x$.
\begin{definition}
%Let $M$ be a Riemannian manifold. 
Let $G=I(\mathcal{M})$ act transitively on $\mathcal{M}$ and
$H=\text{Stab}(o)$, $o\in \mathcal{M}$ (called the ``origin'' of
$\mathcal{M}$) be a subgroup of $G$. Then, $\mathcal{M}$ is a
homogeneous space and can be identified with the quotient space $G/H$
under the diffeomorphic mapping $gH \mapsto g.o, g \in G$
\cite{helgason}.
\end{definition}
In fact, if $\mathcal{M}$ is a homogeneous space, then $G$ is a Lie
group. A Stiefel manifold, $\text{St}(p, n)$ (definition of the
Stiefel manifold is given in next section) is a homogeneous space and
can be identified with $O(n)/O(n-p)$, where $O(n)$ is the group of
orthogonal matrices. Now, we will list some of the
important properties of Homogeneous spaces that will
be used throughout the rest of the paper.  
%\ \\ \ \\ 

{\bf Properties of Homogeneous spaces:} Let $(\mathcal{M},
g^{\mathcal{M}})$ be a Homogeneous space. Let $\omega^{\mathcal{M}}$
be the corresponding volume form and $F: \mathcal{M}\rightarrow
\mathbf{R}$ be any integrable function. Let $g \in G$, s.t. $y = g.x$,
$x,y \in \mathcal{M}$. Then, the following facts are true:
\begin{enumerate}
\label{facts}
\item $g^{\mathcal{M}}(dy, dy) = g^{\mathcal{M}}(dx, dx)$.
\item $d(x, z) = d(y, g.z)$, for all $z \in \mathcal{M}$.
\item $\displaystyle\int_{\mathcal{M}} F(y)\omega^{\mathcal{M}}(x) =
  \int_{\mathcal{M}} F(x)\omega^{\mathcal{M}}(x)$
\end{enumerate}

\begin{definition}
A Riemannian symmetric space is a Riemannian manifold $\mathcal{M}$
with the following property: $(\forall x \in \mathcal{M})(\exists s_x
\in G)$ such that $s_x.x = x$ and $ds_x\at{x}=-I$. $s_x$ is called
    {\it symmetry} at $x$ \cite{helgason}.
\end{definition}
\begin{proposition} \label{prop1}
\cite{helgason} A symmetric space $\mathcal{M}$ is a homogeneous space
with a symmetry, $s_o$, at $o\in \mathcal{M}$. For the other point
$x\in \mathcal{M}$, by transitivity of $G$, there exists $g\in G$ such
that $x=g.o$ and $s_x = g\cdot s_o \cdot g^{-1}$.
\end{proposition}
\begin{proposition}
\label{prop2}
\cite{helgason} Any symmetric space is geodesically complete.
\end{proposition}
Some examples of symmetric spaces include, $\mathbf{S}^n$ (the
hypersphere), $\mathbf{H}^n$ (the hyperbolic space) and $\text{Gr}(p,
n)$ (the Grassmannian). It is evident from the definition that
symmetric space is a homogeneous space but the converse is not
true. For example, the Stiefel manifold is not a symmetric space.
\begin{proposition}
\label{prop3}
\cite{helgason} The mapping $\sigma: g \mapsto s_o \cdot g \cdot s_o$
is an involutive automorphism of $G$ and the stabilizer of $o$, i.e.,
$H$, is contained in the group of fixed points of $\sigma$.
\end{proposition}
Clearly, $\sigma(e) = e$, as $\sigma$ is an automorphism, $e \in G$ is
the identity element. Recall, $G$ is a Lie group, hence,
differentiating $\sigma$ at $e$, we get an involutive automorphism of
the Lie algebra $\mathfrak{g}$ of $G$ (also denoted by
$\sigma$). Henceforth, we will use $\sigma$ to denote the automorphism
of $\mathfrak{g}$. Since $\sigma$ is involutive, i.e., $\sigma^2 = I$,
$\sigma$ has two eigen values, $\pm 1$ and let $\mathfrak{h}$ (Lie
algebra of $H$) and $\mathfrak{p}$ be the corresponding eigenspaces,
then $\mathfrak{g}= \mathfrak{h}+\mathfrak{p}$ (direct sum).
\begin{proposition}
\label{prop4}
\cite{helgason} $[\mathfrak{h}, \mathfrak{h}] \subseteq \mathfrak{h}$,
$[\mathfrak{h}, \mathfrak{p}] \subseteq \mathfrak{p}$ \text{and}
$[\mathfrak{p}, \mathfrak{p}] \subseteq \mathfrak{h}$
\end{proposition}
Hence, $\mathfrak{h}$ is a Lie subalgebra of
$\mathfrak{g}$. Henceforth, we will assume $\mathfrak{g}$ to be
semisimple. We can define a symmetric, bilinear form, $B$ on
$\mathfrak{g}$ as follows $B(u, v) =
\text{trace}\left(\text{ad}(u)\circ \text{ad}(v)\right)$, where
$\text{ad}(u)$ is the adjoint endomorphism of $\mathfrak{g}$ defined by
$\text{ad}(u)(v) = [u, v]$. $B$ is called the {\it Killing form} on
$\mathfrak{g}$.
\begin{definition}
The decomposition of $\mathfrak{g}$ as $\mathfrak{g} =
\mathfrak{h}+\mathfrak{p}$ is called the Cartan decomposition of
$\mathfrak{g}$ associated with the involution $\sigma$. Furthermore,
$B$ is negative definite on $\mathfrak{h}$, positive definite on
$\mathfrak{p}$ and $\mathfrak{h}$ and $\mathfrak{p}$ are orthogonal
complement of each other with respect to $B$ on $\mathfrak{g}$.
\end{definition}
Recall, a symmetric space, $\mathcal{M}$, can be identified with
$G/H$. Note that, $o$, the ``origin'' of $\mathcal{M}$ can be written
as $o = eH$, $e \in G$ is the identity element. Since, $\mathfrak{p}$
can be identified with $T_{o} \mathcal{M}$, the Riemannian metric
$g^\mathcal{M}$ on $\mathcal{M}$ corresponds to the Killing form $B$
on $\mathfrak{p}$ \cite{helgason}, which is a $H$-invariant form.
Without loss of generality, we will assume that $\mathfrak{g}$ is over
$\mathbf{R}$ and $\mathfrak{g}$ be semisimple (equivalently, the
Killing form on $\mathfrak{g}$ is non-degenerate). The symmetric space
$G/H$ is said to be {\it compact} ({\it noncompact}) \emph{iff} the
sectional curvature is strictly positive (negative), equivalently
\emph{iff} $\mathfrak{g}$ is compact (noncompact).
%\ \\ \ \\ 

{\bf Duality: } Given a semisimple Lie algebra $\mathfrak{g}$ with the
Cartan decomposition $\mathfrak{g} = \mathfrak{h}+\mathfrak{p}$,
construct another Lie algebra $\tilde{\mathfrak{g}}$ from
$\mathfrak{g}$ as follows: $\tilde{\mathfrak{g}} =
\mathfrak{h}+J(\mathfrak{p})$, where $J$ is a complex structure of
$\mathfrak{p}$ (real Lie algebra). From the definition
of complex structure, $J: \mathfrak{p}\rightarrow \mathfrak{p}$, is an
automorphism on $\mathfrak{p}$ s.t., $J^2 = -I$. $J$ satisfies the
following equality: $J([T, W]) = [J(T),W]=[T,J(W)]$, for all $T, W \in
\mathfrak{p}$. We will call $\tilde{\mathfrak{g}}$ the dual Lie
algebra of $\mathfrak{g}$.  It is easy to see that if $\mathfrak{g}$
corresponds to a symmetric space of noncompact type,
$\widetilde{\mathfrak{g}}$ is a symmetric space of compact type and
vice-versa. {\it This duality property is very useful and is a key
ingredient of this paper}.  %\ \\ \ \\

Now, we will briefly describe the geometry of two
  Riemannian manifolds, namely the Stiefel manifold and the
  Grassmannian. We need the geometry of Stiefel manifold throughout
  the rest of the paper. Furthermore, observe that, the Stiefel and
  the Grassmannian form a fiber bundle. In order to draw samples from
  a distribution on the Stiefel, we will use the samples drawn from a
  distribution on the Grassmannian by exploiting the fiber bundle
  structure. Hence, we will require the geometry of the Grassmannian
  as well, which we will briefly present below.

%Now, we will briefly describe the geometry of
%three Riemannian manifolds namely, the Stiefel manifold,
%Special Orthogonal group and the Grassmannian. Stiefel manifold is a
%homogeneous space (but not  symmetric space) whereas the other two are symmetric
%spaces. %\ \\ \ \\

{\bf Differential Geometry of the Stiefel manifold:} The set of all
full column rank $(n \times p)$ dimensional real matrices form a
Stiefel manifold, $\text{St}(p, n)$, where $n \geq p$. A compact
Stiefel manifold is the set of all column orthonormal real
matrices. When $p<n$, $\text{St}(p, n)$ can be identified with
$SO(n)/SO(n-p)$, where $SO(m)$ is $m\times m$ special orthogonal
group. Note that, when we consider the quotient space,
$SO(n)/SO(n-p)$, we assume that $SO(n-p) \simeq F(SO(n-p))$ is a
subgroup of $SO(n)$, where, $F: SO(n-p) \rightarrow SO(n)$ defined by
$X\mapsto \begin{bmatrix} I_p & 0 \\ 0 & X\end{bmatrix}$ is an
  isomorphism from $SO(n-p)$ to $F(SO(n-p))$.

\begin{proposition} 
\label{prop5}
$SO(n-p)$ is a closed Lie-subgroup of $SO(n)$. Moreover, the quotient
space $SO(n)/SO(n-p)$ together with the projection map, $\Pi:
SO(n)\rightarrow SO(n)/SO(n-p)$ is a principal bundle with $SO(n-p)$ as
the fiber.
\end{proposition}
\begin{proof}
$SO(n-p)$ is a compact Lie-subgroup of $SO(n)$, hence $SO(n-p)$ is a
closed subgroup. The fiber bundle structure of $(SO(n),
SO(n)/SO(n-p), \Pi)$ follows directly from the closedness of
$SO(n-p)$. As $SO(n)$ is a principal homogeneous space (because
$SO(n) \simeq \text{St}(n-1,n)$ and $SO(n)$ acts on it freely),
hence the principal bundle structure.
\end{proof}

With a slight abuse of notation, henceforth, we denote the compact
Stiefel manifold by $\text{St}(p, n)$. Hence, $\text{St}(p, n) = \{ X
\in \mathbf{R}^{n\times p} | X^T X = I_p \}$, where $I_p$ is the $p
\times p$ identity matrix. The compact Stiefel manifold has dimension
$pn - \frac{p(p+1)}{2}$. At any $X \in \text{St}(p, n)$, the tangent
space $T_{X}\text{St}(p, n)$ is defined as follows $
T_{X}\text{St}(p,n) = \{U \in \mathbf{R}^{n \times p} | X^TU + U^TX =
0 \} $.  Now, given $U, V$ $\in$ $T_{X}\text{St}(p, n)$, the canonical
Riemannian metric on $\text{St}(p, n)$ is defined as follows:
\begin{align}
\label{st:metric}
{\langle U, V \rangle}_{X} = trace\Big(U^TV\Big)
\end{align}
With this metric, the compact
Stiefel manifold has non-negative sectional curvature
\cite{Ziller07examplesof}.

Given $X \in \text{St}(p, n)$, we can define the Riemannian retraction
and lifting map within an open neighbourhood of $X$. We will use an
efficient Cayley type retraction and lifting maps respectively on
$\text{St}(p, n)$ as defined in
\cite{fraikin2007optimization,kaneko2013empirical}. It should be
mentioned that though the domain of retraction is a subset of the
domain of inverse-Exponential map, on $\text{St}(p, n)$ retraction/
lifting is a useful alternative since, \emph{there are no closed form
  expressions for both the Exponential and the inverse-Exponential
  maps on Stiefel manifold}. Recently, a fast iterative algorithm to
compute Riemannian inverse-Exponential map has been proposed in
\cite{zimmermann2017matrix}, which can be used instead of retraction/
lifting maps to compute FM in our algorithm.

In the neighborhood of $[I_p\: 0]$ ($n\times p$ matrix with upper-right
$p\times p$ block is identity and rest are zeros), given $X \in
\text{St}(p, n)$, we define the lifting map $\text{Exp}^{-1}_X $ $:
St(p, n) \rightarrow$ $T_X \text{St}(p,n)$ by $ \text{Exp}^{-1}_X (Y)
= \begin{bmatrix} C & -B^T \\ B & 0
\end{bmatrix}
$ where, $C$ is a $p \times p$ skew-symmetric matrix and $B$ is a
$(n-p) \times p$ matrix defined as follows: $ C = 2(X_u^T+Y_u^T)^{-1}
sk(Y_u^TX_u+X_l^TY_l)(X_u+Y_u)^{-1} $ and $ B =
(Y_l-X_l)(X_u+Y_u)^{-1} $ where, $X = [X_u, X_l]^T$, and $Y = [Y_u,
  Y_l]^T$ with $X_u, Y_u \in \mathbf{R}^{p\times p}$, and $X_l, Y_l
\in \mathbf{R}^{(n-p)\times p}$, provided that $X_u + Y_u$ is
nonsingular. $sk(M)$ is defined as $\frac{1}{2} (M^T -
M)$ and, $Y \in \text{St}(p, n)$.

Furthermore, in the neighborhood of $[I_p\: 0]$, the retraction map
defined above is a diffeomorphism (since it is a chart map) from
$\text{St}(p,n)$ to $\mathfrak{so}(n)$.

\begin{proposition}
\label{prop6}
The projection map $\Pi: SO(n)\rightarrow SO(n)/SO(n-p)$ is a covering
map on the neighborhood of $SO(n-p)$ in $SO(n)/SO(n-p)$.
\end{proposition}
\begin{proof}
First, note that under the identification of $\text{St}(p,n)$ with
$SO(n)/$ $SO(n-p)$, the neighborhood of $[I_p\:\: 0]$ in
$\text{St}(p,n)$ can be identified with the neighborhood of $SO(n-p)$
in $SO(n)/SO(n-p)$. Now, the retraction map defined above is a (local)
diffeomorphism from $\text{St}(p,n)$ to $\mathfrak{so}(n)$. Also, the
Cayley map is a diffeomorphism from $\mathfrak{so}(n)$ to the
neighborhood of $I_n$ in $SO(n)$. Thus, the map $\Pi: SO(n)\rightarrow
SO(n)/SO(n-p)$ is a diffeomorphism to the neighborhood of $SO(n-p)$ in
$SO(n)/SO(n-p)$ (using the fact that the composition of two
diffeomorphisms is a diffeomorphism). Now, since $SO(n)$
is compact and $\Pi$ is surjective, $\Pi$ is a covering map on the
neighborhood of $SO(n-p)$ in $SO(n)/SO(n-p)$ using the following Lemma.
\end{proof}

\begin{lemma}
Under the hypothesis in Proposition \ref{prop6}, $\Pi:
SO(n)\rightarrow SO(n)/SO(n-p)$ is a covering map in the neighborhood
from the neighborhood of $I_n$ in $SO(n)$ to the neighborhood of
$SO(n-p)$ of $SO(n)/SO(n-p)$.
\end{lemma}

\begin{proof}
In Proposition \ref{prop6}, we have shown that $\Pi$ is local
diffeomorphism in the neighborhood specified in the hypothesis. Let
$\mathcal{V}$ be a neighborhood around $SO(n-p)$ in $SO(n)/SO(n-p)$
and $\mathcal{U}$ be a neighborhood around $I_n$ in $SO(n)$ on which
$\Pi$ is a diffeomorphism. Let $Y \in \mathcal{V}$, as $SO(n)/SO(n-p)$
is a $T_2$ space, hence, $\left\{Y\right\}$ is closed, thus,
$\Pi^{-1}(Y)$ is closed and since $SO(n)$ is compact, hence
$\Pi^{-1}(Y)$ is compact. For each $X \in \Pi^{-1}(Y)$, let
$\mathcal{U}_X$ be a open neighborhood around $X$ where $\Pi$
restricts to a diffeomorphism (and hence homeomorphism). Then,
$\left\{\mathcal{U}_X : X \in \Pi^{-1}(Y)\right\}$ is an open cover of
$\Pi^{-1}(Y)$, thus as a finite subcover
$\left\{\mathcal{U}_X\right\}_{X\in I}$, where $I$ is finite. We chose
$\left\{\mathcal{U}_X\right\}$ to be disjoint as $SO(n)$ is a $T_2$
space. Let $\mathcal{W} = \cap_{X \in I} \Pi(\mathcal{U}_X)$ which is
an open neighborhood of $Y$. Then, $\left\{\Pi^{-1}(\mathcal{W}) \cap
\mathcal{U}_X\right\}_{X \in I}$ is a disjoint collection of open
neighborhoods each of which maps homeomorphically to
$\mathcal{V}$. Hence, $\Pi$ is a covering map in the local
neighborhood.
\end{proof}

Given $W \in \mathfrak{so}(n)$, the Cayley map is a conformal mapping,
$Cay : \mathfrak{so}(n)$ $\rightarrow SO(n)$ defined by $ Cay(W) =
(I_n + W) (I_n - W)^{-1} $. Using the Cayley mapping, we can define
the Riemannian retraction map $\text{Exp}_X$ $: T_X \text{St}(p, n)
\rightarrow \text{St}(p, n)$ by $\text{Exp}_X (W) = Cay(W)X$.  Hence,
given $X, Y \in \text{St}(p, n)$ within a regular geodesic ball (the
geodesic ball does not include the cut locus) of appropriate radius
(henceforth, we will assume the geodesic ball to be regular), we can
define the unique geodesic from $X$ to $Y$, denoted by $\Gamma_X^Y(t)$
as
\begin{equation}
\label{st:geo}
\Gamma_X^Y(t) = \text{Exp}_X (t\;\text{Exp}^{-1}_X (Y))
\end{equation}
Also, we can define the distance between $X$ and $Y$ as 
\begin{eqnarray}
\label{st:dis}
d(X, Y) &=& \sqrt{\langle \text{Exp}^{-1}_X (Y), \text{Exp}^{-1}_X (Y) \rangle}.
\end{eqnarray}
%\ \\ \ \\
\begin{comment}
{\bf Differential Geometry of $\mathbf{SO(n)}$:} $SO(n)$ is a compact
Riemannian manifold, hence by the Hopf-Rinow theorem, it is also a
geodesically complete manifold \cite{helgason}. Its
geometry is well understood and we recall a few relevant concepts here
and refer the reader to \cite{helgason} for
details. $SO(n)$ has a Lie group structure and the corresponding Lie
algebra, $\mathfrak{so}(n)$, is defined as, $ \mathfrak{so}(n) = \{W
\in \mathbf{R}^{n\times n} | W^T = -W\} $.  In other words,
$\mathfrak{so}(n)$ (the set of Left invariant vector fields with
associated Lie bracket) is the set of $n \times n$ anti-symmetric
matrices. The Lie bracket, $[,]$, operator on $\mathfrak{so}(n)$ is
defined as the commutator, i.e., for $U, V \in \mathfrak{so}(n)$, $[U,
  V] = UV-VU$. Now, we can define a Riemannian metric on $SO(n)$ as
follows: $ {\langle U, V \rangle}_{X} = trace\Big(U^TV\Big) $ where,
$U, V \in T_X (SO(n))$, $X \in SO(n)$. Note that, it can be shown that
this is a bi-invariant Riemannian metric. Under this bi-invariant
metric, now we define the Riemannian exponential and inverse
exponential map as follows. Let, $X, Y \in SO(n)$, $U \in T_{X}
(SO(n))$. Then, $ Exp_{X}^{-1}(Y) = X \log(X^T Y) $ and $ Exp_{X} (U)
= X \exp(X^T U) $ where, $\exp$, $\log$ are the matrix exponential and
logarithm respectively.  %\ \\ \ \\ 
\end{comment}

{\bf Differential Geometry of the Grassmannian $\text{Gr}(p,n)$:} The
Grassmann manifold (or the Grassmannian) is defined as the set of all
$p$-dimensional linear subspaces in $\mathbf{R}^n$ and is denoted by
$\text{Gr}(p, n)$, where $p \in \mathbf{Z}^{+}$, $n \in
\mathbf{Z}^{+}$, $n \geq p$. Grassmannian is a symmetric space and can
be identified with the quotient space $SO(n)/S\left(O(p)\times
O(n-p)\right)$, where $S\left(O(p)\times O(n-p)\right)$ is the set of
all $n\times n$ matrices whose top left $p\times p$ and bottom right
$n-p \times n-p$ submatrices are orthogonal and all other entries are
$0$, and overall the determinant is $1$. A point $\mathcal{X} \in
\text{Gr}(p, n)$ can be specified by a basis, $X$.  We say that
$\mathcal{X} = \text{Col}(X)$ if $X$ is a basis of $\mathcal{X}$,
where $\text{Col}(.)$ is the column span operator. It is easy to see
that the general linear group $\text{GL}(p)$ acts isometrically,
freely and properly on $\text{St}(p, n)$. Moreover, $\text{Gr}(p, n)$
can be identified with the quotient space $\text{St}(p, n)/
\text{GL}(p)$. Hence, the projection map $\Pi: \text{St}(p, n)
\rightarrow \text{Gr}(p, n)$ is a {\it Riemannian submersion}, where
$\Pi(X) \triangleq \text{Col}(X)$. Moreover, the triplet
$(\text{St}(p,n), \Pi, \text{Gr}(p,n))$ is a fiber bundle.

At every point $X \in \text{St}(p, n)$, we can define the {\it
  vertical space}, $\mathcal{V}_X \subset T_X \text{St}(p, n)$ to be
$\textrm{Ker}(\Pi_{*X})$. Further, given $g^{\text{St}}$, we define
the {\it horizontal space}, $\mathcal{H}_X$ to be the
$g^{\text{St}}$-orthogonal complement of $\mathcal{V}_X$. Now, from
the theory of principal bundles, for every vector field
$\widetilde{U}$ on $\text{Gr}(p, n)$, we define the {\it horizontal
  lift} of $\widetilde{U}$ to be the unique vector field $U$ on
$\text{St}(p, n)$ for which $U_X \in \mathcal{H}_X$ and $\Pi_{*X} U_X
= \widetilde{U}_{\Pi(X)}$, $\text{for all } X \in \text{St}(p, n)$. As, $\Pi$
is a Riemannian submersion, the isomorphism
$\Pi_{*X}|_{\mathcal{H}_X}: \mathcal{H}_X \rightarrow
T_{\Pi(X)}\text{Gr}(p, n)$ is an isometry from $(\mathcal{H}_X,
g^{\text{St}}_X)$ to $(T_{\Pi(X)}\text{Gr}(p, n),
g^{\text{Gr}}_{\Pi(X)})$. So, $g^{\text{Gr}}_{\Pi(X)}$ is defined as:
\begin{align}
\label{gr:metric}
g^{\text{Gr}}_{\Pi(X)}(\widetilde{U}_{\Pi(X)}, \widetilde{V}_{\Pi(X)}) = g^{\text{St}}_{X} (U_X, V_X) = \textrm{trace}((X^TX)^{-1}U_X^TV_X) 
\end{align}
where, $\widetilde{U}, \widetilde{V} \in T_{\Pi(X)}\text{Gr}(p, n)$
and $\Pi_{*X} U_X = \widetilde{U}_{\Pi(X)}$, $\Pi_{*X} V_X =
\widetilde{V}_{\Pi(X)}$, $U_X \in \mathcal{H}_X$ and $V_X \in
\mathcal{H}_X$. 

\section{Gaussian distribution on Homogeneous spaces}
\label{sampling}
%\ \\
In this section, we define the Gaussian distribution,
$\mathcal{N}(\bar{x}, \sigma)$ on a Homogeneous space, $\mathcal{M}$,
$\bar{x} \in \mathcal{M}$ (location parameter), $\sigma >0$ (scale
parameter), and then propose a sampling algorithm to draw samples from
the Gaussian distribution on $\text{St}(p, n)$. Furthermore, we will show
that the maximum likelihood estimator (MLE) of $\bar{x}$ is the
Fr\'{e}chet mean (FM) \cite{frechet1948elements} of the samples.

We define the probability density function, $f(.;\bar{x},\sigma)$ with
respect to $\omega^{\mathcal{M}}$ (the volume form) of the Gaussian
distribution $\mathcal{N}(\bar{x}, \sigma)$ on $\mathcal{M}$ as:
\begin{align}
\label{hom:gauss}
f(x;\bar{x},\sigma) = \frac{1}{C(\sigma)}\exp\left(\ \frac{-d^2(x,\bar{x})}{2\sigma^2} \right)
\end{align}
The above is a valid probability density function, provided the
normalizing factor, $C(\sigma)$ is a constant, i.e., does not depend
on $\bar{x}$ which we will prove next.
\begin{proposition}
\label{prop7}
Let us define $Z(\bar{x}, \sigma) \triangleq
\displaystyle\int_{\mathcal{M}} f(x;\bar{x},\sigma)
\omega^{\mathcal{M}}(x)$. Then, $C(\sigma) = Z(\bar{x}, \sigma) = Z(o,
\sigma)$, where $o \in \mathcal{M}$ is the origin.
\end{proposition}
\begin{proof}
As the group action on $\mathcal{M}$ is transitive, there exists $g
\in G$ s.t., $\bar{x} = g.o$.
\begin{align*}
Z(\bar{x}, \sigma) &=\int_{\mathcal{M}} f(x;\bar{x},\sigma) \omega^{\mathcal{M}}(x) \\
&= \int_{\mathcal{M}} f(g^{-1}.x;g^{-1}.\bar{x},\sigma) \omega^{\mathcal{M}}(x) \:\:\: (\text{using Fact }\ref{facts}\text{ in section }\ref{theory}) \\
&= \int_{\mathcal{M}} f(x; o,\sigma) \omega^{\mathcal{M}}(x) \:\:\: (\text{using Fact }\ref{facts}\text{ in section }\ref{theory}) \\
&= Z(o, \sigma)
\end{align*}
Hence, $C(\sigma) = Z(o, \sigma)$, i.e., does not depend on $\bar{x}$.
\end{proof}

Now that we have a valid definition of a Gaussian distribution,
$\mathcal{N}(\bar{x}, \sigma)$ on a Homogeneous space, we propose a
sampling algorithm for drawing samples from $\mathcal{N}(\bar{X},
\sigma)$ on $\text{St}(p, n)$ (which is a homogeneous space), $\bar{X}
\in \text{St}(p, n), \sigma>0$. 
\subsection{Sampling algorithm}
\label{sampling1}
In order to draw samples from $\mathcal{N}(\bar{X}, \sigma)$ on
$\text{St}(p, n)$, it is sufficient to draw samples from
$\mathcal{N}(O, \sigma)$ where $O \in \text{St}(p,n)$ is the
origin. Then, using group operation, we can draw samples from
$\mathcal{N}(\bar{X}, \sigma)$ for any $\bar{X} \in \text{St}(p,
n)$. We will assume, $O = [I_p\: 0]$ ($n\times p$ matrix with the
upper-right $p\times p$ block being the identity and the rest being
zeros). We will first draw samples from $\mathcal{N}(\mathcal{O},
\sigma)$ on $\text{Gr}(p, n)$, where $\mathcal{O} = \Pi(O)$ and use
this sample to get a sample on $St(p, n)$ using $\mathcal{N}(O,
\sigma)$. Note that $\text{Gr}(p, n)$ is a symmetric space and hence a
homogeneous space and thus we have a valid Gaussian density on
$\text{Gr}(p, n)$ using Eq. \ref{hom:gauss}.

\begin{proposition}
\label{prop8}
Let $\mathcal{X} \sim \mathcal{N}(\mathcal{O}, \sigma)$ where
$\mathcal{O} = \Pi(O)$, $X^TX = I$. Then, $\text{Exp}_{O}(W) \sim
\mathcal{N}(O, \sigma)$, with $W = U\Theta V^T$,
where, $U\Sigma V^T = X(O^TX)^{-1}-O$ and $\Theta = \arctan \Sigma$.
\end{proposition}
\begin{proof}
It is sufficient to show that $d(O, \text{Exp}_{O}(W)) =
d(\mathcal{O}, \mathcal{X})$. Recall, that $(\text{St}(p,n), \Pi,
\text{Gr}(p, n))$ forms a  fiber bundle. Moreover, the
isomorphism, $\Pi_{*X}|_{\mathcal{H}_X}: \mathcal{H}_X \rightarrow
T_{\Pi(X)}\text{Gr}(p, n)$ is an isometry from $(\mathcal{H}_X,
g^{\text{St}}_X)$ to $(T_{\Pi(X)}\text{Gr}(p, n),
g^{\text{Gr}}_{\Pi(X)})$, for all $X \in \text{St}(p, n)$. From,
\cite{absil2004riemannian}, we know that $\Pi_{*O} (W) =
\text{Exp}_{\mathcal{O}}^{-1}(\mathcal{X})$. So,
\begin{align*}
d^2(\mathcal{O}, \mathcal{X}) &= g^{\text{Gr}}_{\mathcal{O}}\left(\text{Exp}_{\mathcal{O}}^{-1}(\mathcal{X}), \text{Exp}_{\mathcal{O}}^{-1}(\mathcal{X}) \right) \\
&= g^{\text{St}}_{O}(W, W) \:\:\: (\text{as } \Pi_{*O} \text{ is an isomorphism and using Eq.}\ref{gr:metric}) \\
&= d^2(O, \text{Exp}_{O}(W)) \:\:\: (\text{using Eq.}\ref{st:dis})
\end{align*}
\end{proof}

Using the Proposition \ref{prop8}, we can generate a sample from
$\mathcal{N}(O,\sigma)$ on $\text{St}(p,n)$, using a sample from
$\mathcal{N}(\mathcal{O},\sigma)$ on $\text{Gr}(p,n)$. We will now
propose an algorithm to draw samples from
$\mathcal{N}(\mathcal{O},\sigma)$ on $\text{Gr}(p,n)$. Recall that
$\text{Gr}(p,n)$ can be identified as $SO(n)/S\left(O(p)\times
O(n-p)\right)$ which is a semisimple symmetric space of compact type
(let it be denoted by $\mathfrak{g} =
\mathfrak{h}+\mathfrak{p}$). Also, recall from section \ref{theory}
that, every compact semisimple symmetric space has a dual semsimple
symmetric space of non-compact type (denoted by
$\widetilde{\mathfrak{g}}=\mathfrak{h}+J(\mathfrak{p}$ ). Here,
$\mathfrak{g} = so(n)$, $\mathfrak{h} = \begin{bmatrix} \bar{U} & 0
  \\ 0 & \bar{V} \end{bmatrix}$, where $\bar{U} \in so(p)$, $\bar{V}
\in so(n-p)$, $\mathfrak{p} = \begin{bmatrix} 0 & \bar{W}
  \\ -\bar{W}^T & 0 \end{bmatrix}$, $\bar{W} \in \mathbf{R}^{p\times
  (n-p)}$. Then, $\widetilde{\mathfrak{g}} = so(p, n-p)$, and the
corresponding Lie group, denoted by $\widetilde{G} = SO(p, n-p)$ (with
a slight abuse of notation we use $SO(p,n-p)$ to denote the identity
component). Here, $SO(p,n-p)$ is the special pseudo-orthogonal group,
i.e., $$ SO(p,n-p) \triangleq \left\{\widetilde{g}\:\:|
\:\:\widetilde{g}I_{(p,n-p)}\widetilde{g}^T = I_{(p,n-p)},
\text{det}(\widetilde{g}) = 1 \right\}
$$, $$ I_{(p,n-p)} \triangleq \text{diag}(\underbrace{1, \cdots, 1}_{p
  \text{ times}}, \underbrace{-1, \cdots -1}_{(n-p) \text{ times}})
$$. Thus, the dual non-compact type symmetric space of
$SO(n)/S\left(O(p)\times O(n-p)\right)$ (identified with $\text{Gr}(p,
n)$) is $SO(p, n-p)/S\left(O(p)\times O(n-p)\right)$.  Recently, in
\cite{saidvemuri}, an algorithm to draw samples from a Gaussian
distribution on symmetric spaces of non-compact type was presented. We
will use the following proposition to get a sample from a Gaussian
distribution on the dual compact symmetric space.

\begin{proposition}
\label{prop9}
Let $\mathcal{X}'\sim \mathcal{N}(\mathcal{O}', \sigma)$. Let
$\mathcal{X}' = \exp(\text{Ad}(\bar{U})\bar{V}).\mathcal{O}'$, where
$\text{Ad}$ is the adjoint representation, $\bar{U} \in \mathfrak{h}$,
$\bar{V} \in J(\mathfrak{p})$. Then, $\mathcal{X} \sim
\mathcal{N}(\mathcal{O}, \sigma)$, where $\mathcal{X} =
\exp(\bar{U})\cdot \exp(\widetilde{V}).\mathcal{O}$ and $\bar{V} =
J(\widetilde{V})$.
\end{proposition}
\begin{proof}
Observe that $\mathcal{O}' = H = \mathcal{O}$. So, it suffices to show that
$d(\mathcal{X}',\mathcal{O}) = d(\mathcal{X}, \mathcal{O})$.
\begin{align*}
d^2(\mathcal{X}',\mathcal{O}) &= B\left(\bar{V}, \bar{V}\right) \:\:\:
(\text{the metric corresponds to Killing form } B \text{ on }
\mathfrak{p}) \\ &= B\left(J(\widetilde{V}), J(\widetilde{V})\right)
\\ &= B\left(\widetilde{V},\widetilde{V}\right) \:\:\: (\text{Killing
  form is invariant under automorphisms}) \\ &= d^2(\mathcal{X},
\mathcal{O})
\end{align*}
\end{proof}

Note that, the mapping $H \times \mathfrak{p} \rightarrow G$ given by
$(h, \exp(\widetilde{V})) \mapsto h\cdot \exp(\widetilde{V})$ is a
diffeomorphism and is used to construct $\mathcal{X}$ from $(\bar{U},
\exp(\widetilde{V}))$. The mapping $\mathcal{X}' =
\exp(\text{Ad}(\bar{U})\bar{V}).\mathcal{O}'$ is called the polar
coordinate transform.  Now, using Propositions \ref{prop8} and
\ref{prop9}, starting with a sample drawn from a Gaussian distribution
on $SO(p,n-p)/S\left(O(p)\times O(n-p)\right)$, we get a sample from
Gaussian distribution on $\text{St}(p,n)$. We would
  like to point out that, we do not have to compute the normalizing
  constant explicitly in order to draw samples, because, in order to
  get samples on $SO(p, n-p)/S\left(O(p)\times O(n-p)\right)$, we can
  draw samples using the Algorithm 1 in \cite{saidvemuri}, which draws
  samples from the kernel of the density.

\subsection{Maximum likelihood estimation (MLE) of $\bar{X}$}
\label{sampling2}
Let, $X_1, X_2, \cdots X_N$ be i.i.d. samples drawn from
$\mathcal{N}(\bar{X}, \sigma)$ with bounded support (described
subsequently) on $\text{St}(p, n)$, for some $\bar{X}\in \text{St}(p,n),
\sigma>0$. Then, by proposition \ref{prop10.5}, the MLE of $\bar{X}$
is the Fr\'{e}chet mean (FM) \cite{frechet1948elements} of
$\{X_i\}_{i=1}^N$. Fr\'{e}chet mean (FM) \cite{frechet1948elements} of
$\{X_i\}_{i=1}^N \subset \text{St}(p,n)$ is defined as follows:
\begin{align}
\label{mean}
M = \argmin_{X \in \text{St}(p,n)} \sum_{i=1}^N d^2(X_i, X)
\end{align}
We define an (open) ``geodesic ball'' of radius $r>0$ to be
$\mathcal{B}(X, r)=\left\{X_i | d(X, X_i)< r\right\}$ s.t., there
exists a length minimizing geodesic between $X$ to any $X_i \in
\mathcal{B}(X, r)$. A ``geodesic ball'' is said to be ``regular''
\emph{iff} $r<\pi/2(\sqrt{\kappa})$, where $\kappa$ is the maximum
sectional curvature. The existence and uniqueness of the Fr\'{e}chet
mean (FM) is ensured \emph{iff} the support of the distribution
$\mathcal{N}(\bar{X}, \sigma)$ is within a regular geodesic ball
\cite{afsari,kendall}.

\begin{proposition}
\label{prop10}
Let $X \in \text{St}(p,n)$, $U, V \in \mathcal{H}_X $, then, $0\leq
\kappa(U,V)\leq 2$
\end{proposition}
\begin{proof}
Let, $\mathcal{X} = \Pi(X)$. Then, there exists a unique $\widetilde{U},
\widetilde{V} \in T_{\mathcal{X}}\text{Gr}(p,n)$ s.t. $\widetilde{U} =
\Pi_{*X} U$, $\widetilde{V} = \Pi_{*X} V$. $0\leq
\kappa(\widetilde{U}, \widetilde{V})\leq 2$
\cite{grassmanncurvature}. Now, using O'Neil's formula
\cite{cheegerebin}, we know that
$$ \kappa(\widetilde{U}, \widetilde{V}) = \kappa(U,V) +
\frac{3}{4}\|\text{vert}_X\left([U,V]\right)\|^2
$$ where, $\text{vert}_X$ is the orthogonal projection operator on
$\mathcal{V}_X$. Clearly as the second term in the above summation is
non-negative and $\kappa(U,V)$ is non-negative (as $\text{St}(p,n)$ is of compact type), the result follows.
\end{proof}

Observe that, the support of $\mathcal{N}(\bar{X}, \sigma)$ as defined
in proposition \ref{prop8} is a subset of
$\mathcal{H}\triangleq\cup_{X}\text{Exp}_X\left(\mathcal{H}_X\right)
\subset \text{St}(p,n)$, $\mathcal{H}$ is an arbitrary union of open
sets and hence is open. Thus, we can give $\mathcal{H}$ a manifold
structure and using the proposition \ref{prop10}, we can say that if
the support of $\mathcal{N}(\bar{X}, \sigma)$ is within a geodesic
ball $\mathcal{B}(\bar{X}, \pi/2(\sqrt{2}))$, FM exists and is
unique. For the rest of the paper, we assume this condition to ensure
the existence and uniqueness of FM.

\begin{proposition}
\label{prop10.5}
Let, $X_1, X_2, \cdots X_N$ be i.i.d. samples drawn from $\mathcal{N}(\bar{X}, \sigma)$ on $\text{St}(p, n)$ (support of $\mathcal{N}(\bar{X}, \sigma)$ is within a geodesic
ball $\mathcal{B}(\bar{X}, \pi/2(\sqrt{2}))$), $\sigma>0$. Then the MLE of $\bar{X}$ is the FM of $\{X_i\}$.
\end{proposition}
\begin{proof}
The likelihood of $\bar{X}$ given the i.i.d. samples $\{X_i\}$ is given by 
\begin{align}
L(\bar{X},\sigma;\{X_i\}_{i=1}^N)= \frac{1}{C(\sigma)}\prod_{i=1}^N \exp\left(\ \frac{-d^2(X_i,\bar{X})}{2\sigma^2} \right),
\end{align}
where $C(\sigma)$ is defined as in Eq. \ref{hom:gauss}. Now,
maximizing log-likelihood function with respect to $\bar{X}$ is
equivalent to minimizing $\sum_{i=1}^N d^2(X_i, \bar{X})$ with respect
to $\bar{X}$. This gives the MLE of $\bar{X}$ to be the FM of
$\{X_i\}_{i=1}^N$ as can be verified using Eq. \ref{mean}.
\end{proof}

\section{Inductive Fr\'{e}chet mean on the Stiefel manifold}
\label{sec2}

In this section, we present an inductive formulation for computing the
Fr\'{e}chet mean (FM) \cite{frechet1948elements,karcher1977riemannian}
on Stiefel manifold.  We also prove the \emph{Weak Consistency} of our
FM estimator on the Stiefel manifold.  %\ \\ \ \\ 

\textbf{Algorithm for Inductive Fr\'{e}chet Mean Estimator} %\ \\ 

Let $X_1$, $X_2$, $\cdots$ be i.i.d. samples drawn from
$\mathcal{N}(\bar{X}, \sigma)$ (whose support is within a geodesic
ball $\mathcal{B}(\bar{X}, \pi/2(\sqrt{2}))$) on $\text{St}(p,
n)$. Then, we define the inductive FM estimator ({\emph StiFME}) $M_k$
by the recursion in Eqs. \ref{eq7}, \ref{eq7.1}. \\
%\begin{wraptable}{r}{7.2cm}
{\centering
\fbox{\begin{minipage}{32em}
\begin{eqnarray}
\label{eq7}
M_1 = X_1 \\
\label{eq7.1} M_{k+1} = \Gamma_{M_k}^{X_{k+1}}(\omega_{k+1})
\end{eqnarray}
\end{minipage}}
}
%\end{wraptable}
\ \\ \ \\ where, $\Gamma_X^Y: [0,1] \rightarrow \text{St}(p,n)$ is the
geodesic from $X$ to $Y$ defined as $\Gamma_X^Y (t) := \text{Exp}_X
(t\text{Exp}_{X}^{-1}(Y))$ and $\omega_{k+1} =
\frac{1}{k+1}$. Eq. \ref{eq7.1} simply means that the ${k+1}^{th}$
estimator lies on the geodesic between the $k^{th}$ estimate and the ${k+1}^{th}$ sample point. This simple inductive estimator can be shown to converge
to the Fr\'{e}chet expectation, i.e., $\bar{X}$, as stated in Theorem
\ref{thm1}.

\begin{theorem}
\label{thm1}
Let, $X_1, X_2 \cdots X_N$ be i.i.d. samples drawn from a Gaussian
distribution $\mathcal{N}(\bar{X}, \sigma)$ on $\text{St}(p, n)$ (with
a support inside a regular geodesic ball of radius
$<\pi/2\sqrt{2}$). Then the inductive FM estimator ({\emph StiFME}) of
these samples, i.e., $M_N$ converges to $\bar{X}$ as $N \rightarrow
\infty$.
\end{theorem}

\begin{proof}
We will start by first stating the following propositions.
\begin{proposition}
\label{prop11}
Using Proposition \ref{prop5}., we know that $\Pi: SO(n) \rightarrow
SO(n)/SO(n-p)$ is a principal bundle and moreover using Proposition
\ref{prop6}, we know that this map is a covering map in the
neighborhood of $SO(n-p)$ in $SO(n)/SO(n-p)$. Let, $g^{SO}$ be the
Riemannian metric on $SO(n)$ and $g^q$ be the metric on the quotient
space $SO(n)/SO(n-p)$. Then, $g^{SO} = \Pi^{*}
g^q$.
\end{proposition}

\begin{proposition} 
\label{prop12}
Let, $X_i = g_iH$, where $H := SO(n-p)$ and $g_i \in G:=SO(n)$. Let,
$M$ is an defined in Eq. \ref{mean}, then, $M = g_MH$, where $g_M =
\argmin_{g \in SO(n)} \sum_{i=1}^N d^2(g_i, g)$.
\end{proposition}
\begin{proof}
Let, $M = \bar{g}H$, for some $\bar{g} \in G$. Then, observe that,
\begin{align*}
d^2(X_i, M) &= d^2(g_iH, \bar{g}H) \\
&= d^2(\bar{g}^{-1}g_iH, H) \:\:\: \text{using property 2 of homogeneous space} \\ 
&= d^2(\bar{g}^{-1}g_i, e) \:\:\: \text{using Proposition \ref{prop11}} \\
&= d^2(g_i, \bar{g}) \:\:\: \text{as $SO(n)$ a Lie group}
\end{align*}
Thus the claim holds.
\end{proof}
By the Proposition \ref{prop12}, we can see that in order to prove
Theorem \ref{thm1}, it is sufficient to show weak consistency on
$SO(n)$. We will state and prove the weak consistency on $SO(n)$ in
the next theorem.
\end{proof}

\begin{theorem}
\label{thm2}
Using the hypothesis in Theorem \ref{thm1}, let $g_1, g_2, \cdots g_N$
be the corresponding i.i.d. samples drawn from the (induced) Gaussian
distribution $\mathcal{N}(\bar{g}, \sigma)$ on $SO(n)$ where $\bar{X}
= \bar{g}H$ ($H :=SO(n-p)$) (it is easy to show using Proposition
\ref{prop12} that this (induced) distribution on $SO(n)$ is indeed a
Gaussian distribution on $SO(n)$).  Then the inductive FM estimator
({\emph StiFME}) of these samples, i.e., $g_N$ converges to $\bar{g}$
as $N \rightarrow \infty$.
\end{theorem}
\begin{proof}
Since $SO(n)$ is a special case of the (compact) Stiefel
manifold, i.e., when $p = n-1$ (as $SO(n)$ can be identified with $\text{St}(n-1,n)$), we will use $X$ instead of $g$ for
notational simplicity.  Let $X \in SO(n)$. Any point in $SO(n)$ can be
written as a product of $n(n-1)/2$ planar rotation matrices by the
following claim.
\begin{proposition}
\label{prop13}
Any arbitrary element of $SO(n)$ can be written as the composition of
planar rotations in the planes generated by the $n$ standard
orthogonal basis vectors of $\mathbf{R}^n$.
\end{proposition}
\begin{proof}
The proof is straightforward. Moreover, each element of $SO(n)$ is a product of $n(n-1)/2$ planar rotations.
\end{proof}

By virtue of the Proposition \ref{prop13}, we can express $X$ as a
product of $n(n-1)/2$ planar rotation matrices. Each planar rotation
matrix can be mapped onto $\mathbf{S}^{n-1}$, hence $\exists \text{
  diffeomorphism } F: SO(n)$ $\rightarrow \underbrace{\mathbf{S}^{n-1}
  \times \cdots \mathbf{S}^{n-1}}_{n(n-1)/2 \text{times}}$. Let's
denote this product space of hyperspheres by $\mathfrak{O}(n-1,
\frac{n(n-1)}{2})$. Then, $F$ is a diffeomorphism from $SO(p)$ to
$\mathfrak{O}(n-1,\frac{n(n-1)}{2})$.  Let $g^{\tiny \mathfrak{O}}$ be
a Riemannian metric on $\mathfrak{O}(n-1, \frac{n(n-1)}{2})$. Let
$\nabla^\mathfrak{O}$ be the Levi-Civita connection on
$T\mathfrak{O}(n-1,\frac{n(n-1)}{2})$. Since, $F$ is a diffeomorphism,
every vector field $U$ on $SO(n)$ pushes forward to a well-defined
vector field $F_{*}U$ on $\mathfrak{O}(n-1,\frac{n(n-1)}{2})$. Define
a map
\begin{align*}
\nabla^{SO}: \Xi(TSO(n)) \times \Xi(TSO(n)) &\rightarrow \Xi(TSO(n)) \\
(U, V) &\mapsto \nabla^{SO}_{U} V
\end{align*},
where $\Xi(TSO(n))$ gives the section of $TSO(n)$.
\begin{proposition}
\label{prop14}
$\nabla^{SO}$ is the Levi-Civita connection on $SO(n)$ equipped with
the pull-back Riemannian metric $F^{*}g^{\tiny \mathfrak{O}}$.
\end{proposition}

\begin{proposition}
\label{prop15}
Given the hypothesis and the notation as above, if
$\gamma$ is a geodesic on $SO(n)$, $F\circ \gamma$ is a geodesic on
$\mathfrak{O}(n-1, \frac{n(n-1)}{2})$.
\end{proposition}
\begin{proof}
Let, $\hat{\gamma} = F\circ \gamma$ be a curve in $\mathcal{O}(n-1,
\frac{n(n-1)}{2})$. Then,
\begin{eqnarray*} 
0 = F_{*} 0 = F_{*} \bigg(\nabla^{SO}_{\gamma'}\gamma'\bigg)&=& F_{*} \bigg(F^{-1}_{*} \bigg(\nabla^{\mathfrak{O}}_{F_{*}\gamma'}F_{*}\gamma' \bigg)\bigg) \\
&=& \nabla^{\mathfrak{O}}_{\hat{\gamma}'} \hat{\gamma}'.
\end{eqnarray*}
Hence, $\hat{\gamma}$ is a geodesic on
$\mathcal{O}(n-1,\frac{n(n-1)}{2})$.
\end{proof}

Now, analogous to Eq. \ref{eq7}, we can define the FM estimator on
$SO(n)$ where the geodesic, $\Gamma_{M_k}^{X_{k+1}}(\omega_{k+1})$ $=
\text{Exp}_{M_k}\bigg(\omega_{k+1}\text{Exp}^{-1}_{M_k}(X_{k+1})\bigg)$. 
Note that, on $SO(n)$,
$\text{Exp}_{M_k}\bigg(\omega_{k+1}\text{Exp}^{-1}_{M_k}(X_{k+1})\bigg)$
$=M_{k} \exp(\omega_{k+1} \log(M_{k}^{-1} X_{k+1}))$.

\begin{proposition}
\label{prop16}
$F_{*} \text{Exp}^{-1}_{M_k}(X_{k+1}) = \text{Exp}^{-1}_{F(M_k)}(F(X_{k+1}))$
\end{proposition}
\begin{proof}
Let $\gamma: [0,1] \rightarrow SO(n)$ be a geodesic from $M_k$ to
$X_{k+1}$. Then, $\text{Exp}^{-1}_{M_k}(X_{k+1}) =
\frac{d}{dt}(\gamma(t))\bigg|_{t=0}$. Using Proposition \ref{prop15},
$F\circ \gamma$ is a geodesic from $F(M_k)$ to $F(X_{k+1})$.
\begin{eqnarray*}
Log_{F(M_k)}F(X_{k+1}) &=& \frac{d}{dt}(F\circ \gamma(t))\bigg|_{t=0} \\
&=& F_{*} \frac{d}{dt}(\gamma(t))\bigg|_{t=0} \\
&=& F_{*} \text{Exp}^{-1}_{M_k}(X_{k+1} )
\end{eqnarray*}
\end{proof}

Let, $\bar{U} = \text{Exp}^{-1}_{F(M_k)}(F(X_{k+1}))$ and $\hat{U} =
\text{Exp}^{-1}_{M_k}(X_{k+1})$. Using Proposition \ref{prop16}, we
get,
\begin{eqnarray*}
g^{\tiny SO}(\hat{U}, \hat{U}) &=& F^{*}g^{\tiny \mathfrak{O}}
(\hat{U}, \hat{U}) \\ 
&=&g^{\mathfrak{O}} (F_{*} \hat{U}, F_{*} \hat{U}) \\ 
&=& g^{\mathfrak{O}} (\bar{U}, \bar{U})
\end{eqnarray*}

Thus, in order to show weak consistency of our proposed estimator on
$\{g_i\} \subset SO(n)$, it is sufficient to show the weak consistency
of our estimator on $\{F(g_i)\} \subset \mathfrak{O}(n-1,
\frac{n(n-1)}{2})$. A proof of the weak consistency of our proposed FM
estimator on hypersphere has been shown in
\cite{salehian2015efficient} (which can be trivially extended to
the product of hyperspheres). This proof of weak consistency on the
hypersphere in turn proves the weak consistency on $SO(n)$.
\end{proof}

Since we have now shown that our proposed FM estimator on
$\text{St}(p,n)$ is (weakly) consistent, we claim that, $\text{Var}(M_N)
\geq \text{Var}(\widehat{M}_N)$ as $N \rightarrow \infty$, where
$\widehat{M}_N$ is the MLE of $\bar{X}$ when
$\left\{X_i\right\}_{i=1}^N$ are i.i.d. samples from
$\mathcal{N}(\bar{X},\sigma)$ on $\text{St}(p,n)$. The following
porposition computes the Fisher information of $\bar{X}$ when samples
are drawn from $\mathcal{N}(\bar{X},\sigma)$ on $\text{St}(p, n)$.

\begin{proposition} 
\label{prop17}
Let $\mathbf{X}$ be a random variable which follows $\mathcal{N}(\bar{X},\sigma)$ on $\text{St}(p,n)$. Then, $I(\bar{X}) = 1/\sigma^2$
\end{proposition}
\begin{proof}
The likelihood of $\bar{X}$  is given by 
\begin{align}
L(\bar{X};\sigma,\mathbf{X}=X)= \frac{1}{C(\sigma)} \exp\left(\ \frac{-d^2(X,\bar{X})}{2\sigma^2} \right)
\end{align}
Then, $I(\bar{X}) = E_\mathbf{X}\left[{\langle \frac{\partial
      l}{\partial \bar{X}}, \frac{\partial l}{\partial \bar{X}}
    \rangle}_{\bar{X}}\right]$, where $l(\bar{X};\sigma,X)$ is the log
likelihood. Now, $l(\bar{X};\sigma,X) =
\frac{\text{Exp}^{-1}_{\bar{X}} X}{\sigma^2}$, hence,
$E_\mathbf{X}\left[{\langle \frac{\partial l}{\partial \bar{X}},
    \frac{\partial l}{\partial \bar{X}}\rangle}\right] =
E_\mathbf{X}\left[{\langle \text{Exp}^{-1}_{\bar{X}} X,
    \text{Exp}^{-1}_{\bar{X}} X \rangle}_{\bar{X}}\right]$ $=
E_\mathbf{X}\left[d^2(X, \bar{X})\right]$.  Now, observe that,
$\text{Var}(\mathbf{X}) = E_\mathbf{X}\left[d^2(X, \bar{X})\right]$
(here, definition of variance of a manifold valued random variable is
as in \cite{pennec_jmiv06}), where from the definition of the Gaussian
distribution, $\text{Var}(\mathbf{X}) = \sigma^2$. Hence, $I(\bar{X})
= 1/\sigma^2$.
\end{proof}

As, $\text{Var}(\widehat{M}_N) = \sigma^2$ (as we have shown that
$\widehat{M}_N$ is the FM of the samples in proposition
\ref{prop10.5}) when the number of samples tends to infinity, and
$\sigma^2= 1/I(\bar{X})$ by proposition \ref{prop17}, we conclude that
MLE achieves the Cram\'er-Rao lower bound asymptotically (this observation
is in line with normal random vector). Furthermore MLE is unbiased,
and is asymptotically an efficient estiamtor. As, we have shown
consistency of our estimator, hence $\text{Var}(M_N)$ is lower bounded
by $\text{Var}(\widehat{M}_N)$ as $N \rightarrow \infty$. In other
words, asymptotically, $\text{Var}(M_N) \geq
\text{Var}(\widehat{M}_N)$ = $\sigma^2$.

\begin{comment}
We now give the convergence rate of our proposed FM estimator, {\it
  StiFME}.
\begin{theorem}
\label{thm3}
\emph{(Convergence rate of StiFME)}: Given hypothesis and notations as above, {\it StiFME} has a linear convergence rate.
\end{theorem}
\begin{proof}
Let $\{X_k\}$ be the samples drawn from a distribution $P(X)$ on $St(p, n)$. Let $M$ be the Fr\'{e}chet Expectation of that distribution. Then, using triangle inequality we have,
\begin{eqnarray*}
d_{St}(M_{k}, M) &\leq& d_{St}(M_{k-1}, M_{k}) + d_{St}(M_{k-1}, M) \\
&=& \frac{1}{k} d_{St}(M_{k-1}, X_k) + d_{St}(M_{k-1}, M) \\
&\leq& \frac{1}{k}\bigg(d_{St}(M_{k-1}, M) + d_{St}(X_k, M)\bigg) + d_{St}(M_{k-1}, M)
\end{eqnarray*}
Hence, 
$$
\frac{d_{St}(M_{k}, M)}{d_{St}(M_{k-1}, M)} \leq (1+\frac{1}{k} + \frac{d_{St}(X_k, M)}{k\;d_{St}(M_{k-1}, M)})
$$ Now, as $d_{St}(M_{k-1}, M)$ and $d_{St}(X_k, M)$ are finite as
$\{X_i\}$ are within a geodesic ball of finite radius, with $k
\rightarrow \infty$, $\frac{d_{St}(M_{k}, M)}{d_{St}(M_{k-1}, M)} \leq
1$. But, the equality holds only if $M$ lies on the geodesic between
$M_{k-1}$ and $X_k$. Let, $l<\infty$ and $M$ lies on the geodesic
between $M_{k-1}$ and $X_{k}$ for some $k > l$. As $M$ is fixed, using
induction, one can easily show that $M$ can not lie on the same
geodesic for all $k > l$. Hence, $r < 1$ and the convergence rate is
linear.

\end{proof}
\end{comment}

\section{Experimental Results} \label{sec3}

In this section, we present experiments demonstrating the performance
of {\it StiFME} in comparison to the batch mode counterpart with
``warm start''(which uses the gradient descent on the sum of squared
geodesic distances cost function, henceforth termed {\it StFME}) on
synthetic and real datasets. By ``warm start" we mean that, when a new
data point is acquired as input, we initialize the FM to its computed
value prior to the arrival/acquisition of the new data point. All the
experimental results reported here were performed on a desktop with a
3.33 GHz Intel-i7 CPU with 24 GB RAM.

\subsection{Comparative performance of {\it StiFME} on Synthetic data}
\label{exp1}
We generated $1000$ i.i.d. samples drawn from a Normal distribution on
$\text{St}(p, n)$ with variance $0.25$ and expectation $\tilde{I}$, where
\begin{equation*}
\tilde{I}_{ij} = \begin{cases}
1 &\text{$1\leq i=j\leq p$}\\
0 &\text{o.w.}
\end{cases}
\end{equation*}
We input these i.i.d. samples to both {\it StiFME} and {\it StFME}. To
compare the performance, we compute the error, which is the distance
(on $\text{St}(p, n)$) between the computed FM and the known true FM
$\tilde{I}$. We also report the computation time for both these
cases. We performed this experiment $5000$ times and report the
average error and the average computation time. The comparison plot
for the average error is shown in Fig. \ref{fig1}, here $n=50$,
$p=10$. In order to achieve faster convergence of {\it StFME}, we used
the ``warm start'' technique, i.e., FM of $k$ samples is used to
initialize the FM computation for ${k+1}$ samples.  From this plot, it
is evident that the average accuracy error of {\it StiFME} is almost
same as that of {\it StFME}.

\begin{figure*}[!ht]
\centering
  \begin{minipage}[b]{0.45\textwidth}
    \includegraphics[width=\textwidth]{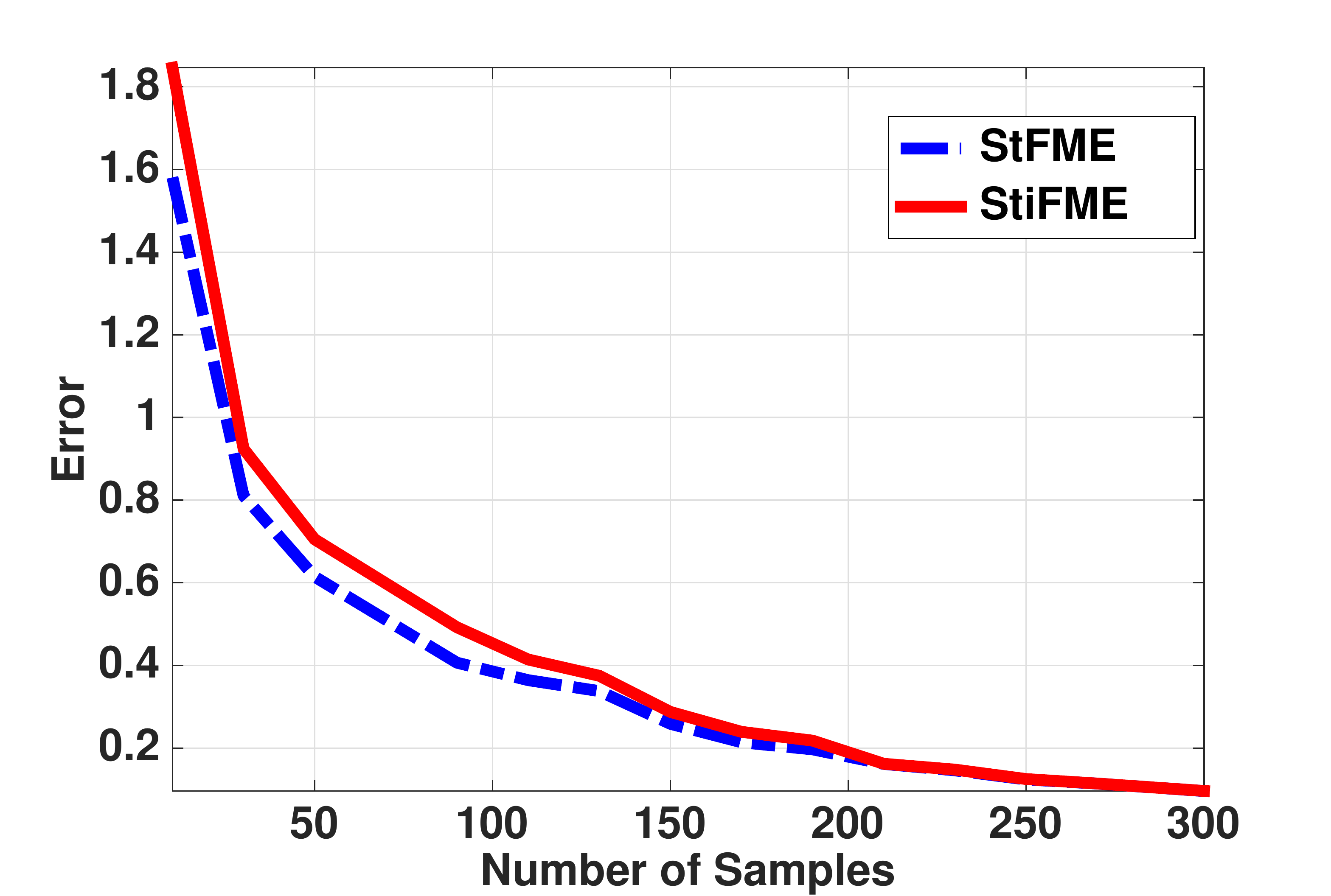}
    \subcaption{Average error.}
    \label{fig1}
  \end{minipage}
  \hfill
  \begin{minipage}[b]{0.45\textwidth}
    \includegraphics[width=\textwidth]{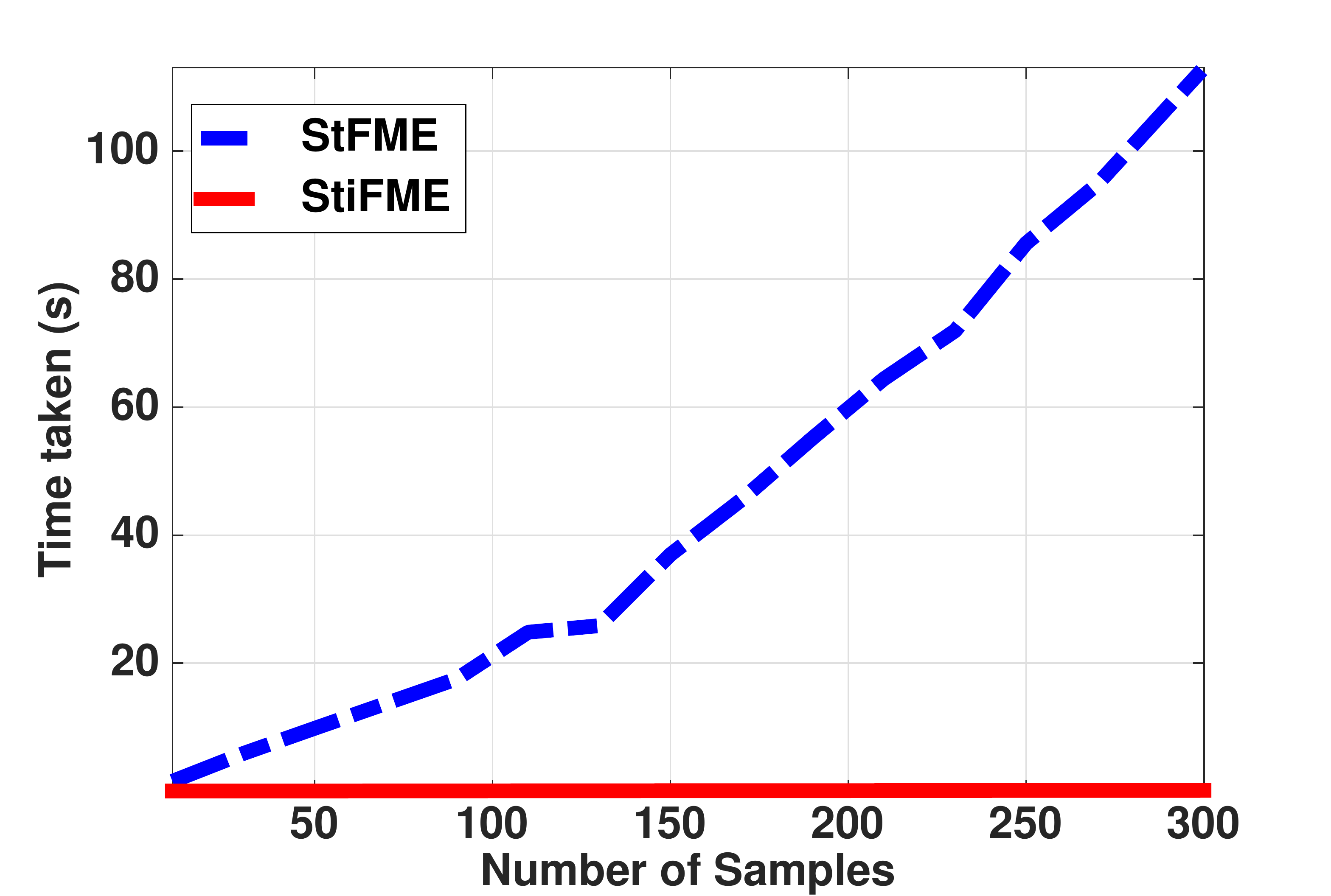}
    \subcaption{Average running time.}
    \label{fig3}
  \end{minipage}
   \begin{minipage}[b]{0.45\textwidth}
    \includegraphics[width=\textwidth]{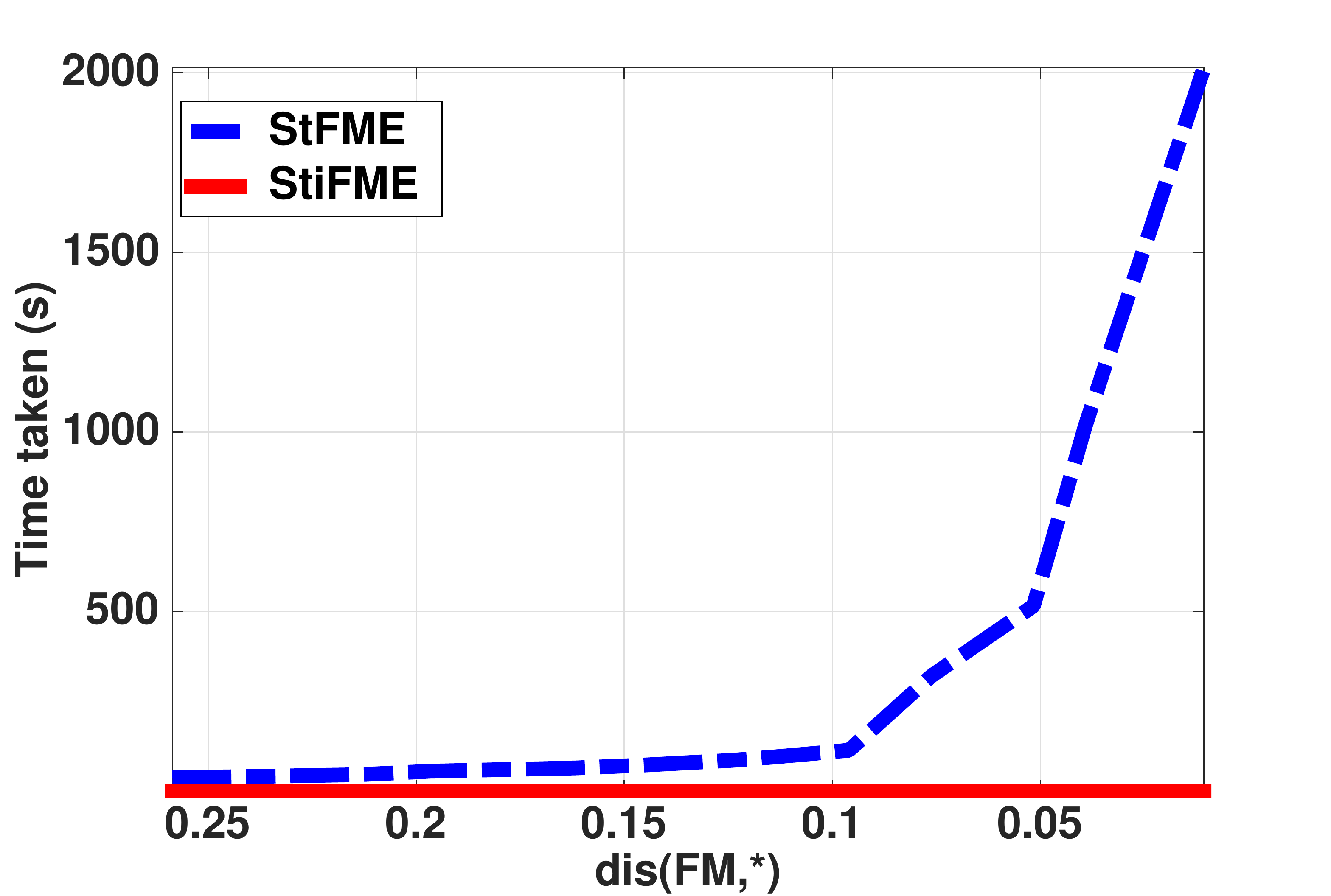}
    \subcaption{Time required to attain a specified accuracy.}
    \label{fig4}
  \end{minipage}
  \caption{Comparison between {\it StFME} and {\it StiFME}. }
  \vspace{-2em}
\end{figure*}

The computation time comparison between {\it StiFME} and {\it
  StFME} is shown in Fig. \ref{fig1}.  From this figure, we can see
that {\it StiFME} outperforms {\it StFME}. As the number of samples
increases, the computational efficiency of {\it StiFME} over {\it StFME}
becomes significantly large. We can also see that the time requirement
for {\it StiFME} is almost constant with respect to the the number of
samples, which makes {\it StiFME} computationally very efficient and
attractive for large number of data samples.

Another interesting question to ask is, how much computation time is
needed in order to estimate the FM with a given error tolerance?  We
answer this question through the plot in Figure \ref{fig4} and present
a comparison of the time required for {\it StiFME} and {\it StFME}
respectively to reach the specified error tolerance.
%% \begin{comment}
%% \begin{figure}
%% \includegraphics[width=0.5\linewidth]{Figures/tol_plot.eps}
%% \caption{The comparison of times required for $StFME$ and
%%   $StiFME$ to attain a specified accuracy.}
%% \label{fig4}
%% \end{figure}
%% \end{comment}
From Fig.\ref{fig4}, is is evident that the time required to reach the
specified error tolerance by {\it StiFME} is far less than that required
by {\it StFME}.

\subsection{Clustering action data from videos}
In this subsection, we applied our FM estimator to cluster the KTH video
action data \cite{schuldt2004recognizing}.
\begin{table}
\begin{center}
%\scalebox{0.55}{
\begin{tabular}{|l|c|c|c|}
\hline
Method & Scenario & Precision(\%) & Time(s) \\
\hline\hline
StFME & d1 & $\mathbf{78.21}$ & $204.59$\\
StiFME & d1 & $77.33$ & $\mathbf{2.32}$ \\
\hline
StFME & d2 & $\mathbf{73.33}$ & $253.15$\\
StiFME & d2 & $70.67$ & $\mathbf{2.48}$ \\
\hline
StFME & d3 & $\mathbf{79.67}$ & $267.40$\\
StiFME & d3 & $77.91$ & $\mathbf{2.59}$ \\
\hline
StFME & d4 & $\mathbf{83.83}$ & $216.27$\\
StiFME & d4 & $90.73$ & $\mathbf{2.82}$ \\
\hline
\end{tabular}
%}
\caption{Comparison results on the \emph{KTH} action recognition database}
\label{tab:tab2}
%\vspace{-0.2in}
\end{center}
\end{table} 
This data contains $6$ actions performed by $25$ human subjects in $4$
scenarios (denoted by `d1', `d2', `d3' and `d4'). From each video, we
extracted a sequence of frames. Then, from each frame we computed the
{\it Histogram of Oriented Gradients} (HOG) \cite{dalal2005histograms}
features. We then used an auto-regressive moving average (ARMA) model
\cite{doretto2003dynamic} to model each activity. The equations for
the ARMA model are given below:
\begin{eqnarray*}
f(t) = Cz(t) + w(t) \\
z(t + 1) = Az(t) + v(t)
\end{eqnarray*}
where, $w$ and $v$ are zero-mean Gaussian noise, $f$ is the feature
vector, $z$ is the hidden state, $A$ is the transition matrix and $C$
is the measurement matrix. In \cite{doretto2003dynamic}, authors
proposed a closed form solution for $A$ and $C$ by stacking feature
vectors over time and performing a singular value decomposition on the
feature matrix. More specifically, let $T$ be the number of frames and
let $F$ be the matrix formed by stacking the feature vectors from each
frame. Let, $U\Sigma V^T$ be SVD of $F$, then, $A$ and $C$ can be
approximated as, $C = U$, $A = \Sigma V^TD_1V
\left(V^TD_2V\right)^{-1}\Sigma^{-1}$, where $D_1$ and $D_2$ are zero
matrices with identity in bottom-left and top-left submatrix
respectively. Clearly, $C$ lies on a Stiefel manifold, but in general
$A$ does not have any special structure. Hence, we identify each
activity with a product space of $U$, $\Sigma$ and $V$. Note that both
$U$ and $V$ lie on Steifel manifold (possibly of different dimensions)
and $\Sigma$ lies in the Euclidean space.

Here, we perform clustering of the actions by doing clustering on the
product manifold of $\text{St}(p, n) \times \text{St}(n, n) \times
\mathbf{R}^n$. The accuracy is reported in Table \ref{tab:tab2}. From
this table, we can see {\it StiFME} depicts significant gain in
computation time over {\it StFME} and is comparable in accuracy.

We would like to point out that in the real data experiment, one can
easily fit a half-normal distribution on
$\left\{d(X_i,\bar{X})\right\}$ by viewing the relation of our
definition of Gaussian distribution with the kernel of the half-normal
distribution on $\left\{d(X_i,\bar{X})\right\}$ with location
parameter $0$ and scale parameter $\sigma^2$. So, the goodness of fit
can be evaluated using the Chi-squared test where the null hypothesis
$H_0$ is that $\left\{d(X_i,\bar{X})\right\}$ are drawn from a
half-normal distribution.

In this experiment, we estimated the goodness of fit in fitting a
Gaussian to the set of samples, $\left\{U_i\right\}$ (samples collected from a given action), using the aforementioned procedure. We found that the
Chi-squared test does not reject the null hypothesis with a $5\%$
significance level, implying that, $\left\{U_i\right\}$ are indeed drawn
from a Gaussian distribution on $\text{St}(p,n)$. We also tried to fit
a Gaussian to the entire data, i.e., over all actions, and found that the entire data are not
drawn from a Gaussian distribution. This is not surprising, as the
entire dataset probably follow a mixture of Gaussians as each
individual action/ cluster follows a Gaussian distribution.
\subsection{Experiments on Vector-cardiogram dataset}

This data set \cite{downs_data} summarises vector-cardiograms of $98$
healthy children aged between $2$-$19$. Each child has two
vector-cardiograms, using the Frank and McFee system respectively. The
two vector-cardiograms are represented as two mutually orthogonal
orientations in $\mathbf{R}^3$, hence, each vector-cardiogram can be
mapped to a point on $\text{St}(2, 3)$.  We perform statistical
analysis via principal geodesic analysis (PGA)
\cite{fletcher2004principal} of the data depicted in figure
\ref{fig55} (at the top). One of the key steps in PGA is to find
the FM, which is depicted in the plot (in black). Further, we
reconstructed the data from the first two principal directions (which
accounts for $>90\%$ of the data variance) and the reconstructed
results are shown in the rightmost plot. The reconstruction error
is on the average $0.05$ per subject, which implies that the reconstruction is quite accurate.

\begin{figure}[!ht]
 	\begin{subfigure}[t]{\textwidth}
        	\centering
        \includegraphics[width=0.47\textwidth]{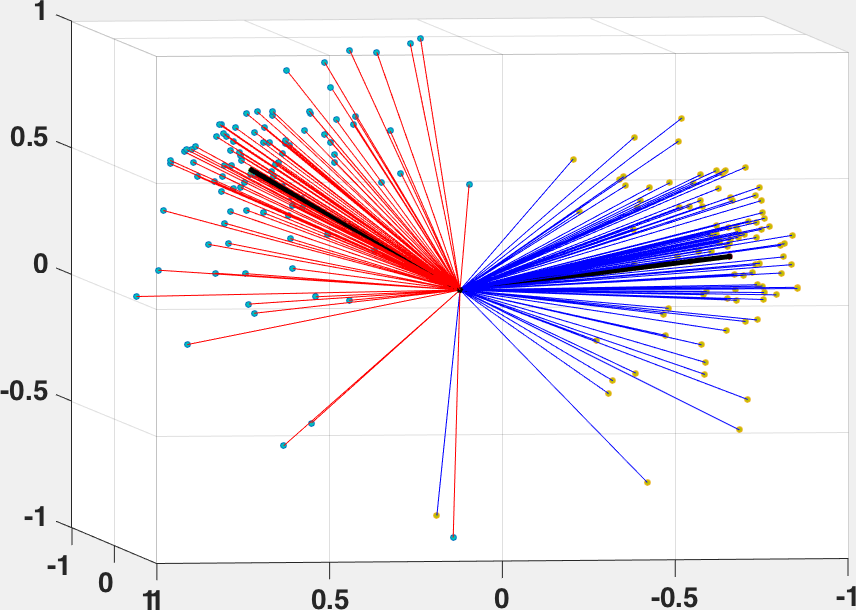}
         \includegraphics[width=0.485\textwidth]{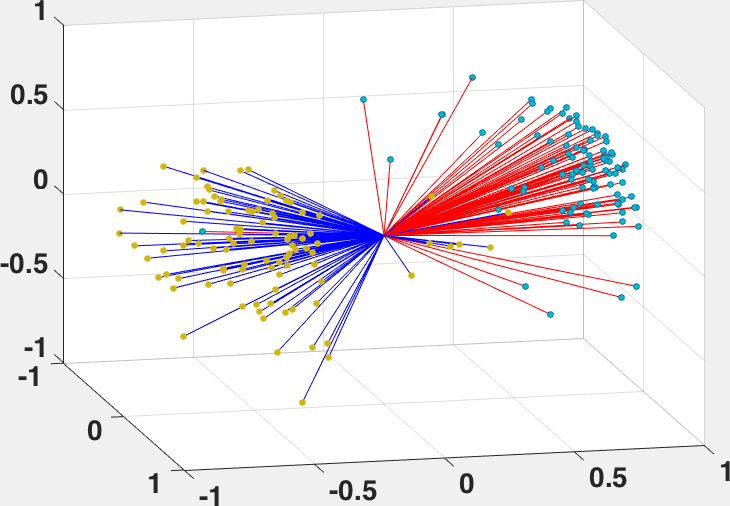}
          \end{subfigure}
   \caption{Averaging on Vector-cardiogram data. {\it Data with FM shown in black} (Left), {\it reconstructed data} (Right)}
\label{fig55}
\end{figure}

\subsection{Comparison with Stochastic Gradient Descent based FM Estimator} \label{exp5}
In this subsection, we present a comparison between {\it StiFME} and
the stochastic gradient descent based FM estimator in
\cite{bonnabel2013stochastic}.

There are two key differences between the algorithm in
\cite{bonnabel2013stochastic} and {\it StiFME}. As in any stochastic
gradient scheme, the next point, i.e., $z_t$ in $w_{t+1} = \exp_{w_t}
(-\gamma_tH(z_t, w_t))$ (Eq.2 in \cite{bonnabel2013stochastic}) is
chosen randomly from the given sample set. Hence, the stochastic
formulation needs several passes over the sample set and reports the
expected value over the passes as the estimated FM. In contrast, {\it
  StiFME} is a deterministic algorithm and hence does not need
multiple passes over the data. Moreover, our selection of this weight
is primarily in spirit the same as the weights in a recursive
arithmetic mean computation in Euclidean space. In contrast,
\cite{bonnabel2013stochastic} does not specify any scheme to choose
the proper step size $\gamma_t$ (Eq.2 of
\cite{bonnabel2013stochastic}). Note that, like in any gradient
descent, the algorithm in \cite{bonnabel2013stochastic} is very much
dependent on a proper step size selection. Step size selection in
gradient descent and its relatives is a hard problem and the most
widely used method (Armijo rule) is computationally expensive. We now
provide two experimental comparisons with algorithm in
\cite{bonnabel2013stochastic}. Consider a data set of $100$ samples
drawn from a Log-Normal distribution, with a small variance of $0.05$
on $\text{St}(10, 50)$. The distance between \{FM and {\it StiFME}\}
and \{FM and computed FM using the algorithm in
\cite{bonnabel2013stochastic}\} (assessed in one pass over the data)
are $0.00025$ and $0.009$ respectively. However,
\cite{bonnabel2013stochastic} requires $19$ passes over the data to
achieve the tolerance of $0.00025$ obtained by {\it StiFME}. For a
larger data variance of $0.29$ on $\text{St}(10, 50)$, the distance
between FM, {\it StiFME} and FM computed from
\cite{bonnabel2013stochastic} are $0.00039$ and $0.03$ (in one pass
over the data) respectively, which is a significant
difference. Furthermore, the method in \cite{bonnabel2013stochastic}
needs $58$ passes over the data to achieve the tolerance achieved by
{\it StiFME}. This clearly indicates better computational efficiency
of {\it StiFME} over the FM estimator in
\cite{bonnabel2013stochastic}.
\subsection{Time Complexity comparison}

The complexity of {\it StFME} is $\mathcal{O}$$(\iota N)$, $N$ is the
number of samples in the data and $\iota$ is the number of iterations
required for convergence. The number of iterations however depends on
the step size used, too small a step size causes very slow convergence
and too large a step size overshoots the FM. In contrast, the
complexity of {\it StiFME} is $\mathcal{O}(N)$ because it outputs the
estimated FM in a single pass through the data. On the other hand the
SGD algorithm proposed in \cite{bonnabel2013stochastic} takes
$\mathcal{O}$$(b\hat{\iota})$, where $b$ is the batch size and
$\hat{\iota}$ is the number of iterations to convergence. So, in
comparison, {\it StiFME} is much faster than the other two competing
algorithms.

\section{Conclusions}\label{conc}
In this paper, we defined a Gaussian distribution on a Riemannian
homogenous space and proved that the MLE of the location parameter of
this Gaussian distribution yields the FM of the samples drawn from the
distribution. Further, we presented a sampling algorithm to draw
samples from this Gaussian distribution on the Stiefel manifold (which
is a homogeneous space) and a novel recursive estimator, {\it StiFME},
for computing the FM of these samples. A proof of weak consistency of
{\it StiFME} was also presented. Further, we also showed that the MLE
of the location parameter of the Gaussian distribution on $St(p,n)$
asymptotically achieves the Cram\'er-Rao lower bound and hence is
efficient. The salient feature of {\it StiFME} is that it does not
require any optimization unlike the traditional methods that seek to
optimize the Fr\'{e}chet functional via gradient descent. This leads
to significant savings in computation time and makes it attractive for
online applications of FM computation for manifold-valued data, such
as clustering etc. We presented several experiments demonstrating the
superior performance of {\it StiFME} over gradient-descent based
competing FM-estimators on synthetic and real data sets.

%% In this paper, we presented a novel recursive estimator, {\it StiFME},
%% for computing the finite sample Fr\'{e}chet mean (FM) of samples drawn
%% from a distribution with finite support and $\mathbf{L}^2$ moment on
%% the Stiefel manifold. We presented a proof of consistency of {\it
%%   StiFME} as well as a proof of linear convergence rate. The salient
%% feature of our estimator is that it does not require any optimization
%% unlike the traditional methods that seek to optimize the Fr\'{e}chet
%% functional via gradient descent. This leads to significant savings in
%% computation time and makes it attractive for online applications of FM
%% computation for manifold-valued data, such as clustering etc. Further,
%% we demonstrated this computational efficiency of our estimator in
%% comparison to the gradient descent based method (with ``warm start'')
%% of computing FM and K-means clustering, on synthetic and real data
%% sets involving, trajectory/path averaging and clustering actions from
%% videos.

\bibliographystyle{plain}
\bibliography{references_used,citations}

\end{document}